\documentclass[letterpaper]{article}
\usepackage{aaai16}
\usepackage{times}
\usepackage{helvet}
\usepackage{courier}

\usepackage{verbatim}
\usepackage{amssymb}
\usepackage{amsmath}
\usepackage{array}
\usepackage{eqnarray}
\usepackage{float} 
\usepackage{graphics,graphicx,epsfig,epstopdf}
\usepackage{psfrag}
\usepackage{subfigure}
\usepackage{tikz}
\usepackage{url}
\usepackage{enumitem}
\usepackage{color}
\usepackage{numprint}
\usepackage[utf8]{inputenc}

\usepackage{amsthm}

\usepackage{thmtools}
\usepackage{thm-restate}

\usepackage{subfigure}
\usepackage{algorithm}
\usepackage[noend]{algpseudocode}
\usepackage[normalem]{ulem}
\usepackage{booktabs}
\usepackage{appendix}

\newtheorem{thm}{Theorem} 

\newtheorem{lem}{Lemma}

\newtheorem{defn}{Definition}

\newtheorem{prob}{Problem}

\usepackage[T1]{fontenc}   

\usepackage{stmaryrd,graphicx}

\usepackage{mathtools}

\makeatletter
\newcommand{\fixed@sra}{$\vrule height 2\fontdimen22\textfont2 width 0pt\shortrightarrow$}
\newcommand{\shortarrow}[1]{%
  \mathrel{\text{\rotatebox[origin=c]{\numexpr#1*45}{\fixed@sra}}}
}

\makeatother

\newcommand{\Rinf}{\mathbb{R}_*}
\newcommand{\Bcal}{\mathcal{B}}
\newcommand{\eps}{\varepsilon}

\renewcommand{\b}{\mathbf}

\newcommand{\ORMV}{\textsc{ORMV}~$(\max, +)$\textsc{-Mul}}
\newcommand{\OMV}{\textsc{OMV}~$(\max, +)$\textsc{-Mul}}
\newcommand{\OMVgeneral}{\textsc{OMV~Mul}}
\newcommand{\runinsec}[1]{\noindent\textbf{#1.}\ }

\makeatletter
\renewenvironment{proof}[1][\proofname]{\par
  \vspace{-.5\topsep}
  \pushQED{\qed}%
  \normalfont
  \topsep0pt \partopsep0pt 
  \trivlist
  \item[\hskip\labelsep
        \itshape
    #1\@addpunct{.}]\ignorespaces
}{%
  \popQED\endtrivlist\@endpefalse
  \addvspace{6pt plus 6pt} 
}
\makeatother

\newif\ifshowsupplementalmaterial
\showsupplementalmaterialtrue

\begin{document}

\title{Decoding Hidden Markov Models Faster Than Viterbi\\
Via Online Matrix-Vector $(\max, +)$-Multiplication}

%

\author{Massimo Cairo\\
Department of Mathematics\\
University of Trento\\
Trento, Italy\\
massimo.cairo@unitn.it
\And 
Gabriele Farina\\
Dipartimento di Elettronica, Informazione e Bioingegneria\\
Politecnico di Milano\\
I-20133, Milan, Italy\\
gabriele2.farina@mail.polimi.it
\AND Romeo Rizzi\\
Department of Computer Science\\
University of Verona\\
Verona, Italy\\
romeo.rizzi@univr.it}

\maketitle
 
\setlength{\textfloatsep}{2mm}

\newcommand{\bigoh}{\ensuremath{\mathcal{O}}}
\newcommand{\todo}[1]{\textcolor{red}{#1}}

\newcommand*{\textcircle}{}
\DeclareRobustCommand*{\textcircle}[1]{\tikz[baseline=(char.base)]{
  \node[shape=circle,draw,inner sep=0.5pt] (char) {\small #1};}}

\newcommand*{\textcirclebf}{}
\DeclareRobustCommand*{\textcirclebf}[1]{\tikz[baseline=(char.base)]{
  \node[shape=circle,draw,inner sep=0.6pt,fill=black] (char) %
    {\small\textcolor{white}{\textbf{#1}}};}}

\begin{abstract}
\begin{quote}
In this paper, we present a novel algorithm for the maximum \emph{a posteriori} decoding (\textsc{mapd}) of time-homogeneous Hidden Markov Models (HMM), improving the worst-case running time of the classical Viterbi algorithm by a logarithmic factor.
In our approach, we interpret the Viterbi algorithm as a repeated computation of matrix-vector $(\max, +)$-multiplications. On time-homogeneous HMMs, this computation is \emph{online}: a matrix, known in advance, has to be multiplied with several vectors revealed one at a time.
Our main contribution is an algorithm solving this version of matrix-vector $(\max,+)$-multiplication in subquadratic time, by performing a polynomial preprocessing of the matrix. Employing this fast multiplication algorithm, we solve the \textsc{mapd} problem in $\bigoh(mn^2/\log n)$ time for any time-homogeneous HMM of size $n$ and observation sequence of length $m$, with an extra polynomial preprocessing cost negligible for $m>n$.
To the best of our knowledge, this is the first algorithm for the \textsc{mapd} problem requiring subquadratic time per observation, under the {only} assumption -- usually verified in practice -- that the transition probability matrix does not change with time.

\end{quote} 
\end{abstract}

\section{Introduction}
Hidden Markov Models (HMMs) are simple probabilistic models originally introduced~\cite{viterbi1967error} to decode convolutional codes. 
Due to their universal and fundamental nature, these models have successfully been applied in several fields, with many important applications, such as gene prediction~\cite{haussler1996generalized}, speech, gesture and optical character recognition~\cite{gales1998maximum,huang1990hidden,starner1998real,agazzi1993hidden}, and part-of-speech tagging~\cite{kupiec1992robust}.
Their applications to bioinformatics began in the early 1990 and soon exploded to the point that currently they hold a recognized place in that field~\cite{yoon2009hidden,makinen2015genome}.

%

A HMM describes a stochastic process generating a sequence of observations $y_1,y_2,\dots,y_n$. Internally, a sequence of hidden states $x_1,x_2,\dots,x_n$ is generated according to a \emph{Markov chain}. At each time instant $t=1,2,\dots,n$, a symbol $y_t$ is observed according to a probability distribution depending on $x_t$.
We consider only time-homogeneous HMMs, i.e.~models whose parameters do not depend on the time $t$. While this assumption covers the majority of applications, some notable exceptions involving time-inhomogeneous models are known \cite{lafferty2001conditional}.

\runinsec{Maximum \emph{a posteriori} decoding (\textsc{mapd})}
Since the states of the model are hidden, i.e.~only the generated symbols can be observed, a natural problem associated with HMMs is the \textsc{mapd} problem:
\textit{given a HMM $\mathcal{M}$ and an observed sequence of symbols $Y$ of length $m$, find any state path $X$ through $\mathcal{M}$ maximizing the joint probability of $X$ and $Y$. We call any such $X$ a \emph{most probable state path} explaining the observation $Y$.}
Traditionally, the \textsc{mapd} problem is solved by the Viterbi algorithm~\cite{viterbi1967error}, in $\bigoh(mn^2)$ time and $\bigoh(mn)$ memory for any model of size $n$ and observation sequence of length $m$.

Over the years, much effort has been put into lowering the cost of the Viterbi algorithm, both in terms of memory and of running time.
\cite{grice1997reduced} showed that a checkpointing technique can be employed to reduce the memory complexity to $\bigoh(\sqrt{m}\cdot n)$; refinements of this idea (embedded checkpointing) deliver a family of time-memory tradeoffs, culminating into an $\bigoh(n\log m)$ memory solution with a slightly increased running time $\bigoh(mn^2\log m)$.

At the same time, several works reducing the time complexity of the algorithm in the average-case were developed \cite{vsramek2007line,churbanov2008implementing,felzenszwalb2004fast,esposito2009carpediem,kaji2010efficient}. 
Many of these works make assumptions on the structure of $\mathcal{T}$, and may lose the optimality or degenerate to the worst case $\Theta(mn^2)$ operations when these assumptions are not fulfilled.
%
%
{In \cite{lifshits2009speeding}, the authors show a method to speed up the decoding of HMMs by a $\bigoh(\log m) $ factor, by precomputing all possible observation sequences of length $\log m$, in a fashion similar to the Four Russians method. This requires the number of such sequences to be ``small''. In the same work, the authors also show that it is possible to compress the observation sequence to achieve speedups proportional to the compression ratios. However, this latter method seems to require the observation sequence to be available in advance.}

To the best of our knowledge no algorithm achieving a worst-case running time better than $\bigoh(mn^2)$ is known under the only assumption of time-homogeneousness.


\runinsec{Approach}
We give an algorithm solving the \textsc{mapd} problem for time-homogeneous HMMs with time complexity asymptotically lower than $\bigoh(mn^2)$, \emph{in the worst case}.
We regard the \textsc{mapd} problem as an iterated computation of a matrix-vector multiplication. For time-homogeneous models, the matrix is known in advance and does not change with time. However, the sequence of vectors to be multiplied cannot be foreseen, as each vector depends on the result of the previous computation; this rules out the possibility to batch the vectors into a matrix and defer the computation. 
We call this version of the problem, in which a fixed matrix has to be multiplied with several vectors revealed one at a time, ``the online matrix-vector multiplication (\OMVgeneral) problem''.

Consider the problem of multiplying a $n \times n$ matrix with a column vector of size $n$. Without further assumptions, the trivial $\bigoh(n^2)$ time algorithm is optimal, since all the $n^2$ elements of the matrix have to be read at least once. However, under the assumption that the matrix is known in advance and can be preprocessed, this trivial lower bound ceases to hold.
Algorithms faster than the trivial quadratic one are known for the \OMVgeneral\ problem over \emph{finite} semirings~\cite{williams2007matrix}, as well as over real numbers with standard $(+, \cdot)$-multiplication, if the matrix has only a \emph{constant} number of distinct values~\cite{liberty2009mailman}. 

However, none of the above algorithm can be applied to time-homogeneous HMMs, as their decoding relies on online real matrix-vector $(\max, +)$-multiplication (\ORMV).
In the specific case of real $(\max, +)$-multiplication, subcubic algorithms have been known for years \cite{dobosiewicz1990more,chan2008all,chan2015speeding} for the \emph{matrix-matrix} multiplication problem, with important applications to graph theory and boolean matrix multiplication, among others.
However, we are not aware of any algorithm solving the \ORMV\ problem in subquadratic time.
Note that the \ORMV\ can be used to compute the \OMVgeneral\ over the Boolean semiring: for this problem, it has been conjectured \cite{Henzinger2015} that no ``truly polynomially subquadratic'' algorithm\footnote{That is, running in time $\bigoh(n^{2-\eps})$ for some $\eps > 0$ after a polynomial preprocessing or the matrix.} exists for the \ORMV\ problem.

We reduce the \ORMV\ problem to a multi-dimensional geometric dominance problem, following an approach similar to that of \cite{bremner2006necklaces,chan2008all}.
Then, the geometric problem is solved by a divide-and-conquer algorithm, which can be regarded as a transposition of the algorithm of \cite{chan2008all} to the online setting.
%
Our technique yields a worst-case $\bigoh(mn^2/\log n)$ algorithm, called \textsc{gdfv}, solving the \textsc{mapd} problem after a polynomial preprocessing of the model.

\runinsec{Contributions} Our key contributions are as follows:
	(i) we extend the geometric dominance reporting problem introduced in \cite{chan2008all} to the online setting;
	(ii) we solve the \ORMV\ problem in $\bigoh(n^2/\log n)$ time after a polynomial preprocessing of the $n\times n$ matrix;
	(iii) we show an algorithm solving the \textsc{mapd} problem on time-homogeneous HMMs in $\bigoh(mn^2/\log n)$ time in the worst-case, after a polynomial preprocessing of the model.

Finally, we experimentally evaluate the performance of our algorithms, with encouraging results. Currently the problem sets in which we outperform Viterbi are limited, but we hope that the approach we propose will open the way to further improvements on this problem in future works.

\section{Preliminaries}\label{sec:preliminaries}
\subsection{Notation}
The $i$-th component of a vector $\b{v}$ is denoted by $\b{v}[i]$; similarly, $\b{M}[i,j]$ denotes the entry of row $i$ and column $j$, in matrix $\b{M}$. Indices will always be considered as starting from $1$.
Given two vectors $\b{a}$ and $\b{b}$ of dimension $n$, such that $\b{a}[i] \le \b{b}[i]$ for every coordinate index $i = 1, \ldots, n$, we write $\b{a} \preceq \b{b}$ and say that $\b{b}$ \emph{dominates} $\b{a}$, or, equivalently, that $(\b{a}, \b{b})$ is a \emph{dominating pair}.

Given a matrix or vector $\b{M}$ with non-negative entries, we write $\log \b{M}$ to mean the matrix or vector that is obtained from $\b{M}$ by applying the logarithm on every component. We will almost always work with the extended set $\Rinf = \mathbb{R}\cup\{-\infty\}$, so that we can write $\log 0 = -\infty$. We assume that $-\infty + x = x + (-\infty) = -\infty$ and $x \ge -\infty$ for all $x\in\Rinf$.

\subsection{Hidden Markov Models (HMMs)}
We formally introduce the concept of time-homogeneous Hidden Markov Models.
\begin{defn}
	\label{defn:HMM}
	A time-homogeneous HMM is a tuple $\mathcal{M} = (\mathcal{S}, \mathcal{A}, \Pi, \mathcal{T}, \mathcal{E})$, composed of:
	\begin{itemize}[nolistsep]
		\item a set $\mathcal{S}=\{s_1, \dots, s_n\}$ of $n$ hidden states; $n$ is called the \emph{size} of the model,
		\item an output alphabet $\mathcal{A} = \{a_1,\dots, a_{|\mathcal{A}|}\}$,
		\item a probability distribution vector $\Pi = \{\pi_1, \dots, \pi_n\}$ over the initial states,
		\item a matrix $\mathcal{T} = \{t_s(s')\}_{s,s'\,\in\,\mathcal{S}}$ of transition probabilities between states,
		\item a matrix $\mathcal{E} = \{e_s(a)\}_{s\,\in\,\mathcal{S}}^{a\,\in\,\mathcal{A}}$ of emission probabilities.
	\end{itemize}
	Matrices $\mathcal{T}$ and $\mathcal{E}$ are stochastic, i.e., the entries of every row sum up to 1.
\end{defn}
%
%
For notational convenience, we relabel the states of a HMM with natural numbers, i.e.~we let $\mathcal{S}=\{1,\dots,n\}$.

As stated in the introduction, HMMs define generative processes over the alphabet $\mathcal{A}$. The initial state $x_0 \in \mathcal{S}$ is chosen according to the distribution $\Pi$; then, at each step, a symbol $a$ is generated according to the probability distribution $e_{x}(a)$, where $x$ is the current state; a new state $x'$ is chosen according to the probability distribution induced by $t_{x}(x')$, and the process repeats.
The probability of a state path $X = (x_1,\dots,x_m)$ joint to an observation sequence $Y = (y_1, \dots, y_m)$ is computed as:
\[
	\Pr(X,Y) = \pi_{x_1}\left(\prod_{i = 1}^{m-1} t_{x_i}(x_{i+1})\right)\left(\prod_{i = 1}^m e_{x_i}(y_i)\right).
\]

\subsection{The Viterbi algorithm}\label{subsec:viterbi}
The Viterbi algorithm consists of two phases: in the first phase, a simple dynamic programming approach is used to determine the probability of the most probable state path ending in each possible state. In the second phase, the data stored in the dynamic programming table is used to reconstruct a most probable state path.
\begin{defn}
	Assume given a HMM  $\mathcal{M} = (\mathcal{S}, \mathcal{A}, \Pi, \mathcal{T}, \mathcal{E})$ and an observed sequence $A = (a_1, \dots, a_m)$. For every $s\in\mathcal{S}$ and $i=1,\dots, m$, denote by $q_i(s)$ the probability of any most probable path ending in state $s$ explaining the observation $A_{i-1} = (a_1, \dots, a_{i-1})$.
\end{defn}
By definition of $q_i(s)$, any most probable path explaining $A$ has probability $\max_{s\in\mathcal{S}}\left\lbrace e_{s,a_m}\cdot q_m(s) \right\rbrace$.
The $q_i(s)$ values can be computed inductively. Indeed, $q_1(s) = \pi_s$ for all $s\in\mathcal{S}$, while for every $i > 1$ and $s\in\mathcal{S}$ it holds:
\begin{equation}\label{eq:qis}
	q_i(s) = \max_{s'\,\in\,\mathcal{S}} \left\lbrace q_{i - 1}(s')\cdot t_{s'}(s)\cdot e_{s'}(a_{i-1})\right\rbrace.
\end{equation}
In order to compute all the $n$ values $q_i(s)$, for any fixed $i > 1$, $\Theta(n^2)$ comparisons have to be performed. This phase is in fact the bottleneck of the algorithm.

The second phase of the algorithm uses the $q_i(s)$ values to reconstruct an optimal path in $\bigoh(mn)$ time. We will not deal with this second and faster part, and only mention that most of the previously developed solutions for it, including the memory saving ones~\cite{vsramek2007line,churbanov2008implementing}, are still applicable once the first part has been carried out based on our approach.

\section{Online matrix-vector multiplication problem}



We formally introduce the online matrix-vector $(\max, +)$-multiplication problem briefly discussed in the introduction.

\begin{defn}[Matrix-vector $(\max, +)$-multiplication]\label{defn:max plus}~\\ Given a $m\times n$ matrix $\b{A}$ and a $n$-dimensional column vector $\b{b}$ over $\Rinf$, their $(\max, +)$-multiplication $\b{A} * \b{b}$ is a column vector of dimension $m$ whose $i$-th component is:
\[
	(\b{A} * \b{b})[i] = \max_{j\,=\,1}^n \left\lbrace\b{A}[i, j] + \b{b}[j]\right\rbrace.
\]
\end{defn}

\begin{prob}[\ORMV\ problem]\label{prob:max plus} Given a $m\times n$ matrix $\b{A}$ over $\Rinf$, perform a polynomial-time preprocess of it, so that the product $\b{A} * \b{b}$ can be computed for any input vector $\b{b}$ of size $n$.
\end{prob}

Lemma~\ref{lem:product to viterbi} is a simple result that bridges the gap between the \textsc{mapd} problem and the \ORMV\ problem, showing that any fast solution to Problem~\ref{prob:max plus} can be turned into a fast solution for the \textsc{mapd} problem.

\begin{lem}\label{lem:product to viterbi}
	Any algorithm for Problem~\ref{prob:max plus} computing $\b{A} * \b{b}$ in $T(m, n)$ time after $U(m, n)$ preprocessing time can be used to solve the \textsc{mapd} problem for any time-homogeneous HMM of size $n$ and observation sequence of length $m$ in $\bigoh(m\cdot T(n,n))$ time after a $U(n,n)$ time preprocessing.
\end{lem}
\begin{proof}
	Denote by $^t\mathcal{T}$ the transpose of the state transition matrix $\mathcal{T}$ of $\mathcal{M}$. Furthermore, for every $i = 1,\dots, m$, introduce the pair of vectors
	\[
		\b{q}_i = (q_i(1), \dots, q_i(n)),\quad \b{e}_i = (e_1(a_{i}), \dots, e_n(a_{i})).
	\]
	The value $q_i(s)$ corresponding to instant $i>0$ and state $s\in \mathcal{S}$ can be computed as follows, in logaritmic scale:
		\[
			\begin{aligned}
				\log q_i(s) &=\log \max_{s'\,\in\,\mathcal{S}}\left\lbrace q_{i-1}(s')\cdot t_{s'}(s)\cdot e_{s'}(a_{i-1}) \right\rbrace\\
				&=\max_{s'\,\in\,\mathcal{S}}\left\lbrace \log(t_{s'}(s)) + \log(q_{i-1}(s')\, e_{s'}(a_{i-1})) \right\rbrace\\
				&=((\log{^t\mathcal{T}}) * (\log \b{q}_{i-1} + \log \b{e}_{i-1}))[s].
			\end{aligned}			
		\]
	

	\noindent Notice that the $n \times n$ matrix $\log {^t\mathcal{T}}$ depends only on the model and is time-invariant; therefore, we can compute ${q}_i(s)$ for all $s=1,\dots,n$ in batch with an instance of \ORMV:
	\[
		\log \b{q}_i = (\log{^t\mathcal{T}}) * (\log \b{q}_{i-1} + \log \b{e}_{i-1}).
	\]
	
	Once the $q_i(s)$ values have been computed, we can use the second part of the Viterbi algorithm out of the box. The time required for the multiplication -- the bottleneck of the algorithm -- is $\bigoh(T(n,n))$ by hypothesis, hence the final algorithm has time complexity $\bigoh(m\cdot T(n,n))$. The time complexity of the preprocessing is $U(n,n)$.
\end{proof}

\section{From multiplication to geometric dominance}

Consider the following geometric problem, which apparently has no relation with the $(\max,+)$-multiplication, nor with the Viterbi algorithm.

\begin{prob}[Online geometric dominance reporting]\label{prob:dominance}~\\ Let $\Bcal$ be a set of $d$-dimensional vectors\footnote{The coordinates of the vectors can range over any chosen totally-ordered set, as long as any two coordinate values can be compared in constant time.}. Given a vector $\b{p} \in \Bcal$, we define its \emph{domination set} as
\[
	\delta_\Bcal(\b{p}) = \{\b{b} \in \Bcal: \b{b} \preceq \b{p}\}.
\]
Preprocess $\Bcal$ so that at a later time the set $\delta_\Bcal(\b{p})$ can be computed for any input vector $\b{p}$.
\end{prob}


\begin{lem}\label{lem:dominance to product}
	Any algorithm solving Problem~\ref{prob:dominance} in $\bigoh(T(d,|\Bcal|)+|\delta_\Bcal(\b{p})|)$ time after a preprocessing time $\bigoh(U(d,|\Bcal|))$ can be turned into an algorithm solving the \ORMV\ problem for any $m\times t$ matrix in $\bigoh(m+t\cdot T(t,m))$ time, with preprocessing time $\bigoh(mt^2 + t\cdot U(t,m))$.
\end{lem}
\begin{proof}
	This constructive proof comes in different stages.
	First, the result is obtained under the simplifying assumptions that (i) neither $\b{A}$ nor $\b{b}$ have any $-\infty$ entry, and (ii) the maximum sum on any row, $\b{A}[\cdot, j] + \b{b}[j]$, is achieved by exactly one value of the column index $j$.
	Later, these conditions will be dropped.
	
	
	Observe that $(\b{A}*\b{b})[i] = \b{A}[i,j^*] + \b{b}[j^*]$ iff
	\begin{equation}\label{eq:inequality}
		\b{A}[i, j^*] + \b{b}[j^*] \geq \b{A}[i, j] + \b{b}[j]
	\end{equation}
	for every column index $j$.
	Under assumption (i), this inequality can be rewritten as
	\[
		\b{A}[i, j] - \b{A}[i, j^*] \leq \b{b}[j^*] - \b{b}[j].
	\]
	Defining the values
	\[
		a_{i,j^*}(j)=\b{A}[i, j]-\b{A}[i, j^*],\quad b_{j^*}(j)=\b{b}[j^*] - \b{b}[j]
	\]
	for all the feasible values of $i, j, j^*$, we obtain
	\begin{equation}\label{eq:forall condition}
		(\b{A}*\b{b})[i] = \b{A}[i,j^*] + \b{b}[j^*] \iff a_{i,j^*}(j) \leq b_{j^*}(j)\quad \forall j.
	\end{equation}
	Notice that the last expression is actually a statement of geometric dominance, between the two $t$-dimensional vectors $\b{\tilde{A}}_{i, j^*} = (a_{i, j^*}(1), \dots, a_{i, j^*}(t))$ and $\b{\tilde{b}}_{j^*} = (b_{j^*}(1), \dots, b_{j^*}(t))$.
	This immediately leads to the following algorithm.
	
	\begin{algorithm}
		\small
		\caption{\small $m\times t$ matrix-vector $(\max,+)$-multiplication}
		\label{algo:tablet algo simplified}
		\begin{algorithmic}[1]
			\Procedure{preprocess}{$\b{A}$}
				\For{$j^* = 1,\dots,t$}
					\For{$i=1,\dots,m$}
						\State $\b{\tilde{A}}_{i,j^*} \gets (a_{i,j^*}(1), \dots, a_{i,j^*}(t))$
					\EndFor{}
					\State $\Bcal_{j^*} \gets \{\b{\tilde{A}}_{1,j^*}, \dots, \b{\tilde{A}}_{m,j^*}\}$
					\State Preprocess $\Bcal_{j^*}$ \label{line:preprocess}
				\EndFor{}	
			\EndProcedure{}
		\end{algorithmic}
		\vspace{1mm}\color{black!15!white}\hrule\color{black}\vspace{1.7mm}
		\begin{algorithmic}[1]
			\Procedure{multiply}{$\b{b}$}\Comment{Returns $\b{A} * \b{b}$}
				\For{$j^* = 1,\dots,t$}
					\State $\b{\tilde{b}}_j \gets (b_{j^*}(1), \dots, b_{j^*}(t))$
					\State $\delta \gets \delta_{\Bcal_{j^*}}(\b{\tilde{b}}_{j^*})$, as defined in Problem~\ref{prob:dominance} \label{line:dominance}
					\For{\textbf{all} $\b{\tilde{A}}_{i,j^*} \in \delta$}
						\State $\b{m}[i] \gets \b{A}[i,j^*]+\b{b}[j^*]$ \label{line:amortized}
					\EndFor{}
				\EndFor{}
				\State \textbf{return} $\b{m}$
			\EndProcedure{}
		\end{algorithmic}
	\end{algorithm}
	Line~\ref{line:preprocess} in procedure~\textsc{preprocess} requires $U(t,m)$ time, making the total preprocessing cost of \textsc{multiply} $\bigoh(mt^2+t\cdot U(t,m))$. As for the running time, Line~\ref{line:dominance} of \textsc{multiply} takes $\bigoh(T(t,m)+|\delta|)$ time, and is executed $t$ times.
	Under assumption (ii), there is only one column $j^*$ that satisfies Equation~\ref{eq:forall condition} for each row $i$, hence the total number of elements appearing in $\delta$ is exactly $m$. As a consequence, Line~\ref{line:amortized} of \textsc{multiply} is executed $m$ times and the total time complexity of \textsc{multiply} is $\bigoh(m+t\cdot T(m,t))$.
	

	We now relax assumptions~(i) and (ii) by applying a transformation of the input. Instead of working on $\Rinf$, we work on triples over $\mathbb{N} \times \mathbb{R} \times \mathbb{N}$.
	The matrix $\b{A}$ is transformed as follows: each element $x>-\infty$ is replaced by $\langle 0, x, 0 \rangle$, while each occurence of $-\infty$ is replaced by $\langle -1, 0, 0 \rangle$.
	The input vector $\b{b}$ is transformed similary, but the third coordinate is used to hold the index of the replaced element. Namely, element $\b{b}[j] = x$ is replaced by $\langle 0, x, j \rangle$ if $x > -\infty$ and by $\langle -1, 0, j \rangle$ otherwise. One can informally regard the first two entries of each triple $\langle a, b, \cdot \rangle$ as a shorthand for the value $a\cdot\infty + b$. Any two triples are compared according to lexicographical order, while addition and subtraction are performed element-wise.
	
	The crucial observation is that, if $\b{A}[i, j^*] + \b{b}[j^*] > \b{A}[i, j] + \b{b}[j]$ before the transformation, then the same holds also after the transformation. Hence, we can solve the transformed problem to obtain the solution of the original problem.
	Our algorithm can be applied as it is to the transformed problem: indeed, the inequality in Equation~\ref{eq:inequality} can be rearranged without any further assumption; moreover, there can be no two distinct columns $j$ and $j'$ achieving the maximum, as the two triples $\b{A}[i, j] + \b{b}[j]$ and $\b{A}[i, j'] + \b{b}[j']$ differ at least on the third element.
	Once the output vector is obtained, replace each triple $\langle k, x, j \rangle$ with $x$ if $k = 0$ and with $-\infty$ otherwise. Conveniently, the third element $j$ holds the index of the column achieving the maximum.
\end{proof}


Lemma~\ref{lem:splice} shows that every fast algorithm for the \OMV\ of narrow rectangular matrices can be turned into a \OMV\ algorithm for square matrices.

\begin{lem}\label{lem:splice}
	Any algorithm computing the \OMV\ of a $m\times t$ matrix in $T(m,t)$ time and $U(m,t)$ preprocessing time, can be used to multiply any $m\times n$ matrix, $n\geq t$, in $\bigoh(n/t\cdot (T(m,t)+m))$ time and $\bigoh(n/t\cdot U(m,t))$ preprocessing time.
\end{lem}
\begin{proof}
	Assume without loss of generality that $n$ is an integer multiple of $t$ (otherwise, add columns to $\b{A}$ and elements to $\b{b}$ with  value $-\infty$ until the condition is met).
	The idea is to split $\b{A}$ and $\b{b}$ into $n/t$ blocks, each of size $m\times t$ and $t\times 1$ respectively:
	\[
		\b{A} = (\b{A}_1 | \cdots | \b{A}_{n/t}),\quad\b{b} = (\b{b}_1 | \cdots | \b{b}_{n/t}).
	\]
	Observe that $
			(\b{A}*\b{b})[i] 
			= \max_{\ell = 1}^{n/t} \left\{ (\b{A}_\ell * \b{b}_\ell)[i] \right\}.
	$
	This immediately leads to the following algorithm. First, we preprocess each block $\b{A}_\ell$ in $U(m,t)$ time with the given algorithm, so that the product $\b{A}_\ell * \b{b}_\ell$ can be later computed in $T(m,t)$ time. As soon as the vector $\b{b}$ is received, compute $\b{m}_\ell = \b{A}_\ell * \b{b}_\ell$ for all $\ell = 1,\dots, n/t$, and finally the output vector by $(\b{A}*\b{b})[i] = \max_{\ell = 1}^{n/t} \left\{ \b{m}_\ell[i] \right\}$.
	
	The time analysis is straightforward. The computation of each $\b{m}_\ell$ takes $T(m,t)$ time. There are $n/t$ such computations and merging the results takes $\bigoh(m\cdot n/t)$ time, yielding a total time of $\bigoh(n/t\cdot (T(m,t)+ m))$. The total preprocessing time is $\bigoh(n/t\cdot U(m,t))$. 
\end{proof}
	Before continuing with the next section, where the main result of this paper will be discussed, we state the following theorem, whose proof is available in the Appendix of this paper.
	\begin{restatable}{thm}{thmexpdominance}\label{thm:expdominance}
		Problem~\ref{prob:dominance} can be solved in $\bigoh(d\log |\Bcal| + |\delta_\Bcal(\b{p})|)$ time per query, after a $\bigoh(|\Bcal|^{d+1})$ time and space preprocessing.
	\end{restatable}

\section{An $\bigoh(n^2/\log n)$ algorithm (\textsc{gdfv})}
\begin{lem}\label{lem:online geometric dominance}
	There exists an algorithm solving Problem~\ref{prob:dominance}
	in $\bigoh(d\,c_\eps^{d} \, |\Bcal|^{\eps} + |\delta_\Bcal(\b{p})|)$ time for every $\eps \in (0,1]$, where $c_\eps \coloneqq 1 / (2^\eps -1)$.
	The preprocessing requires $\bigoh({c'_\eps}^d\,|\Bcal|^{1+\eps})$ time and memory for every $\eps \in (0,\log_2 3/2]$, where $c'_\eps \coloneqq 1/(2^{1+\eps}-2)$.
\end{lem}
\begin{proof}
	We develop a simple divide-and-conquer algorithm; we assume without loss of generality that $|\Bcal|$ is a power of two.
	
	\runinsec{Overview} If $d = 0$, return $\delta_\Bcal(\b{p}) = \Bcal$. If $\Bcal$ contains only one vector $\b{b}$, check if $\b{b} \preceq \b{p}$ and return either $\{\b{b}\}$ or the empty set accordingly.
	In all the other cases, split $\Bcal$ into two sets $\Bcal^-$ and $\Bcal^+$ of size $|\Bcal|/2$, according to the median $d$-th coordinate $\gamma$ of the vectors in $\Bcal$, so that $\b{b^-}[d] \leq \gamma \leq \b{b^+}[d]$ for all $\b{b^-}\in \Bcal^-$ and $\b{b^+}\in \Bcal^+$.
	Now consider the $d$-th coordinate of the query vector $\b{p}$: if it is strictly less than $\gamma$, then all the vectors in $\Bcal^+$ do not occur in the solution. Hence, solve the problem recursively on $\Bcal^-$. Otherwise, both the sets $\Bcal^+$ and $\Bcal^-$ need to be considered; however, the $d$-th coordinate for the vectors in $\Bcal^-$ is known to be $\leq \b{p}[d]$ and can be dropped. Hence, solve the problem recursively on $\Bcal^+$ and $\b{p}$ and on $(\Bcal^-)'$ and $\b{p}'$, where the apostrophe denotes the discard of the last coordinate, and merge the solutions.
	The recursive step is summarized by the following recurrence:
	\[
		\delta_\Bcal(\b{p}) =
		\begin{cases}
			\delta_{\Bcal^-}(\b{p}) & \text{if }\b{p}[d] < \gamma,\\
			\delta_{\Bcal^+}(\b{p}) \cup \delta_{(\Bcal^-)'}(\b{p}') & \text{otherwise.}
		\end{cases}
	\]

	In order to make the algorithm faster, we exploit the fact that $\Bcal$ is known in advance. At preprocessing time, we build a tree that guides the execution of the algorithm, where each node $u$ corresponds to a subproblem $\Bcal_u$ over a $d_u$-dimensional space. The root corresponds to the original set $\Bcal$. If $|\Bcal_u| \geq 2$ and $d_u \geq 1$, then the node $u$ stores the median value $\gamma$ and has three children, corresponding to the subproblems $\Bcal_u^+$, $\Bcal_u^-$ and $(\Bcal_u^-)'$. Otherwise, $u$ is a leaf storing the content of $\Bcal_u$.
	We analyze the cost of building the tree later on. For now, notice that the size of the tree is at most polynomial in $|\Bcal|$: the height is at most $\log |\Bcal|$, as the value $|\Bcal_u|$ halves at each level, so the nodes are at most $\bigoh(3^{\log_2 |\Bcal|})=\bigoh(|\Bcal|^{\log_2 3})=\bigoh(|\Bcal|^{1.59})$.
	
	\runinsec{Time analysis} Our algorithm starts from the root node, and visits recursively the nodes in the tree that are needed to solve the problem.
	When we reach a leaf $u$ with $d_u=0$, we output $\Bcal_u$ (which is not empty) in $\bigoh(|\Bcal_u|)$ time.
	If instead $d_u>0$ and $|\Bcal_u|=1$, we pay $\bigoh(d)$ time to check if $\b{b} \preceq \b{p}$, and $\bigoh(1)$ to output $\b{b}$ if needed.
	On internal nodes, we only pay constant extra time as the median coordinate $\gamma$ is known from the tree.
	The cost of producing the output is $\bigoh(|\delta_\Bcal(\b{p})|)$, and is measured separately.
	Hence, the running time is $\bigoh(T_d(|\Bcal|) + |\delta_\Bcal(\b{p})|)$ where $T_d(n)$ satisfies the linear recurrence relation
	\begin{equation}\label{eqn:geometric running time}
		\begin{aligned}
			T_d(n) &= 1 + \max
			\begin{cases}
				T_d(n/2) \\
				T_{d-1}(n/2) + T_d(n/2)
			\end{cases}
			\\
			&= T_{d-1}(n/2) + T_d(n/2) + 1,
		\end{aligned}
	\end{equation}
	with base cases $T_d(1) = d$ and $T_0(n) = 0$. (The time required to handle this last case is included in $\bigoh(|\delta_\Bcal(\b{p})|)$).
	We show by induction that
	\[
		T_d(n) \leq \overline{T}_d(n) \coloneqq d\,c_\eps^d\,n^\eps,
	\]
	for any chosen $\eps\in (0,1]$, where $c_\eps \coloneqq 1/(2^\eps-1)$.
	Notice that $c_\eps \geq 1$ for $\eps\in (0,1]$, thus the statement is true for the base cases as $\overline{T}_d(n)\geq d$.
	Assuming the inductive hypothesis, we obtain for $n\geq 2$ and $d \geq 1$:
	\[
		\begin{aligned}
			T_d(n) &\leq \overline{T}_d(n/2)+\overline{T}_{d-1}(n/2)+1 \\
			&= d\,c_\eps^d\,(n/2)^\eps + (d-1)\,c_\eps^{d-1}\,(n/2)^\eps + 1 \\
			&\leq d\,c_\eps^d\,(n/2)^\eps + d\,c_\eps^{d-1}\,(n/2)^\eps \\
			&= d\,c_\eps^d \, n^\eps \, \frac{1+c_\eps^{-1}}{2^\eps}
			= \overline{T}_d(n)\,\frac{1+2^\eps-1}{2^\eps} = \overline{T}_d(n),
		\end{aligned}
	\]
	completing the induction. Thus, the time complexity of the algorithm is $\bigoh(d\,c_\eps^d\,|\Bcal|^\eps + |\delta_\Bcal(\b{p})|)$.

	\runinsec{Preprocessing} The tree is built starting from the root. Finding the median $d$-th coordinate $\gamma$, computing $\Bcal_u^+$, $\Bcal_u^-$ and $(\Bcal_u^-)'$, and storing the data in the node $u$, all require $\bigoh(|\Bcal_u|)$ time and memory. Hence, the time and memory cost to build the tree is $\bigoh(U_d(|\Bcal|))$, where $U_d(n)$ satisfies the recurrence
	\[
		U_d(n) = 2U_d(n/2) + U_{d-1}(n/2) + n
	\]
	with base cases $U_0(n)=n$ and $U_d(1)=1$.
	We show by induction that
	\[
		U_d(n) \leq \overline{U}_d(n) \coloneqq 3 {c'_\eps}^d\,n^{1+\eps} - 2n
	\]
	for any chosen $\eps\in (0,\log_2 3/2]$, where $c'_\eps \coloneqq 1/(2^{1+\eps}-2)$.
	Notice that $c'_\eps \geq 1$ for $\eps\in (0,\log_2 3/2]$. Hence, the statement is true for the base cases, as $\overline{U}_d(n)\geq 3{c'_\eps}^d\,n^{1+\eps} - 2n\geq 3n-2n \geq n$.
	Assuming the inductive hypothesis, we obtain for $n\geq 2$ and $d \geq 1$:
	\[
		\begin{aligned}
			U_d(n) &\leq 2\,\overline{U}_d(n/2) + \overline{U}_{d-1}(n/2) + n\\
			&= 2 \cdot (3\,{c'_\eps}^d\,(n/2)^{1+\eps} - n) \\
			& \qquad + (3\,{c'_\eps}^{d-1}\,(n/2)^{1+\eps} - n) + n\\
			&= 2\cdot 3\,{c'_\eps}^d\,(n/2)^{1+\eps} + 3\,{c'_\eps}^{d-1}\,(n/2)^{1+\eps} - 2n\\
		\end{aligned}
	\]\[
		\hspace{-1.4cm}\begin{aligned}
			&= 3\,{c'_\eps}^d\, n^{1+\eps} \cdot \frac{2+{c'_\eps}^{-1}}{2^{1+\eps}} - 2n\\
			&= 3\,{c'_\eps}^d\, n^{1+\eps} - 2n = \overline{U}_d(n).
		\end{aligned}
	\]
	completing the induction. Hence, the time and memory cost of the preprocessing phase is $\bigoh(U_d(|\Bcal|))=\bigoh({c'_\eps}^d \, n^\eps)$.
\end{proof}

We remark that the time complexity and the recurrence of the simple algorithm given in the above proof differ from those of Chan by necessity, as the online setting requires each vector to be treated separately. On the other hand, his result follows directly from ours:

\begin{thm}[\cite{chan2008all}, Lemma 2.1]
	Given $n$ red/blue points in $\mathbb{R}_*^d$ we can report all $K$ dominating pairs in
	$\bigoh(k_\eps^d \,n^{1+\eps} + K)$ time for any $\eps \in (0,1)$, where $k_\varepsilon \coloneqq 2^\eps/(2^\eps - 1)$.
\end{thm}
\begin{proof}
	After reading the set $\Bcal$ of blue points, preprocess them as described in the proof of Lemma~\ref{lem:online geometric dominance}. Then, for each red point $\b{p}$, perform a query to find all the dominators of $\b{p}$, i.e. $\delta_\Bcal(\b{p})$; simply flush out the union of all the dominating pairs obtained. By Lemma~\ref{lem:online geometric dominance}, the cost of the preprocessing is $\bigoh((2^{1+\eps}-2)^{-d} \,n^{1+\eps}) = \bigoh(k_\eps^d \,n^{1+\eps})$. On the other hand, each of the $n$ queries takes time $\bigoh(d(2^\eps - 1)^{-d}\, n^\eps)$, that is\footnote{Indeed, using the binomial expansion formula we have:\[k_\eps^d = \left(1+ \frac{1}{2^\eps - 1}\right)^d \ge d\left(\frac{1}{2^\eps - 1}\right)^{d-1} = \Omega( d\,c_\eps^d ).\]\vspace{-5mm}} $\bigoh(k_\eps^d\,n^{\eps})$, excluding the output; the overhead due to the actual output of the pairs is $\bigoh(K)$. The final cost of the algorithm is therefore $\bigoh(c_\eps^d \,n^{1+\eps} + K)$ as desired.
\end{proof}

Finally, Lemmas~\ref{lem:product to viterbi}, \ref{lem:online geometric dominance}, and \ref{lem:splice} combine into the following.

\begin{thm}\label{thm:gdfv}
	There exists an algorithm solving the \textsc{mapd} problem in $\bigoh(mn^2/\log n)$ time after a polynomial preprocessing of the model, for any HMM of size $n$ and observation sequence of length $m$.
\end{thm}
\begin{proof}
	It is enough to show how to solve the \ORMV\ problem for any $n\times n$ matrix in $\bigoh(n^2/\log n)$ time. To this end, apply Lemma~\ref{lem:splice} with $m = n$ and $t = \alpha\log_2 n$, where $\alpha \in (0, 1/2)\subseteq \mathbb{R}$. This gives a running time of $\bigoh(n\cdot T(\alpha \log_2 n, n) + n^2/\log n)$, where $T(d, n)$ is the cost of computing the domination set of a $d$-dimensional vector over a fixed set of $n$ points (see Problem~\ref{prob:dominance}). Substituting the bounds of Lemma~\ref{lem:online geometric dominance} into the time complexity, yields a polynomial preprocessing cost, and the following time bound for each multiplication:
	\[
		\begin{aligned}
			\bigoh((2^\eps - 1)^{-\alpha \log_2 n} \alpha n^{1+\eps} \log^2 n + n^2/\log n)=\\
			=\bigoh(n^{1 + \eps-\alpha\log_2(2^\eps - 1)}\log^2 n + n^2/\log n)
		\end{aligned}
	\]
	for all $\eps \in (0, 1)$. Setting $\eps = 2\alpha$, the exponent of the first term becomes $1 + \alpha(2 - \log_2(4^\alpha - 1)) < 2$ for all $0 < \alpha < 1/2$. Therefore, the time complexity of the algorithm is $\bigoh(n^2/\log n)$.
\end{proof}
We call the resulting algorithm \emph{geometric dominance faster Viterbi} (\textsc{gdfv}).

\section{Experimental evaluation}
\runinsec{Methodology} All the algorithms are implemented in the C++11 language, compiled using the \texttt{clang} compiler and run on the OSX 10.10.3 operating system. The main memory is a 8GB 1600MHz DDR3 RAM, and the processor is an Intel Core i7-4850HQ CPU, with 6MB shared L3 cache.
All the matrices and vectors used for the experiments have entries sampled from a uniform distribution over $(0, 1] \subseteq \mathbb{R}$.

\runinsec{Results} \emph{How does our proposed \ORMV\ algorithm for narrow matrices compare to the trivial one?} We analyze the throughput of Algorithm~\ref{algo:tablet algo simplified}, based on the geometric subroutines exposed in the proof of Lemma~\ref{lem:online geometric dominance}, comparing it with the trivial multiplication approach. For every chosen pair $(n,t)$, we run 25 tests, each of which consists of an online multiplication of a $n\times t$ matrix with $10\,000$ vectors. The results of our tests are summarized in Figure~\ref{fig:throughput mul}, where we see that our algorithm can be up to $4$ times faster than the trivial one. This is mainly due to the fact that the number of accesses to the tree is much less than $n\cdot t$, and to the lower number of comparisons needed to find the answer. See the Appendix for a more thorough analysis of the average number of accesses to the decision trees. 

\begin{figure}[h]
	\includegraphics[width=\linewidth]{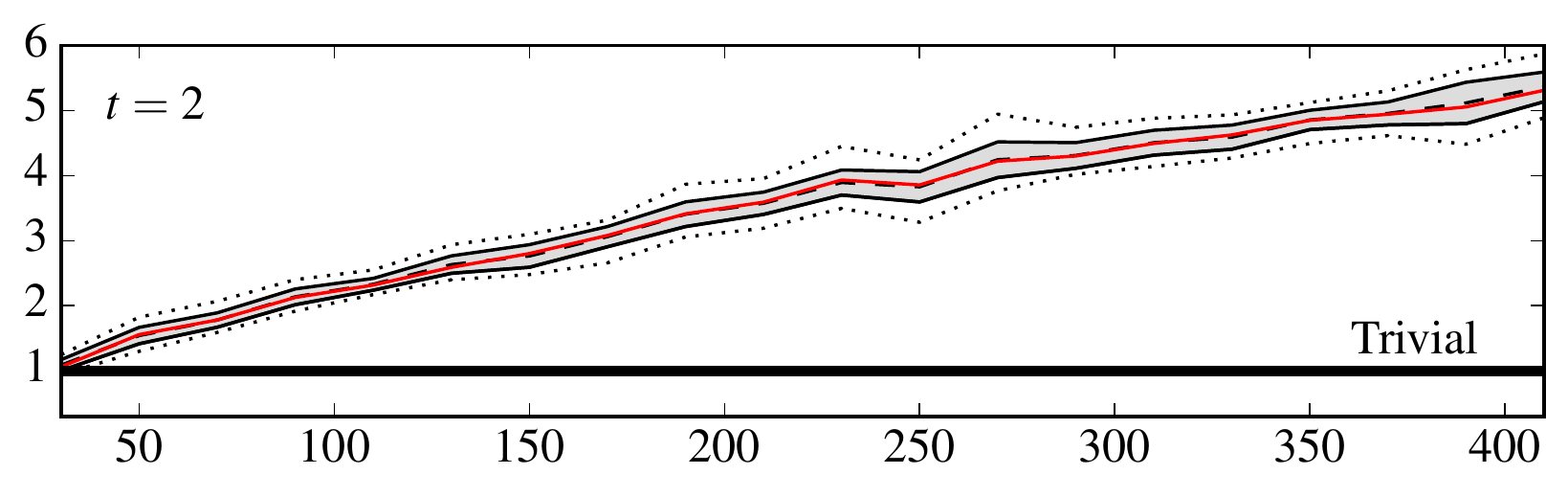}\\[-2.2mm]
	\includegraphics[width=1.03\linewidth]{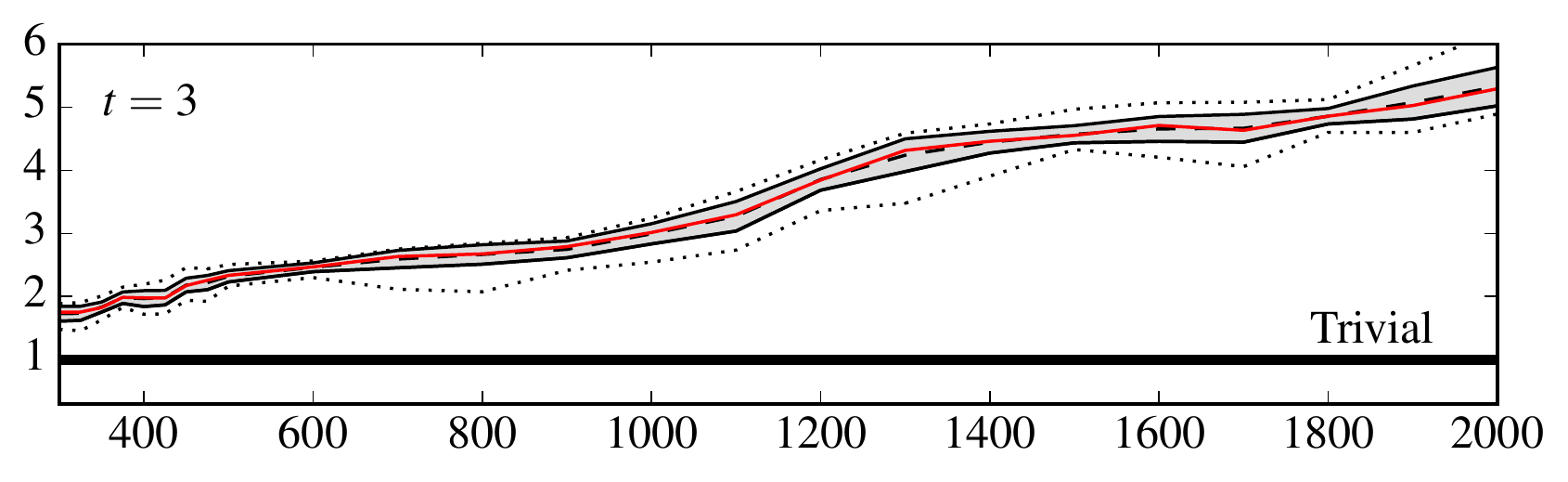}\\[-2.2mm]
	\includegraphics[width=1.038\linewidth]{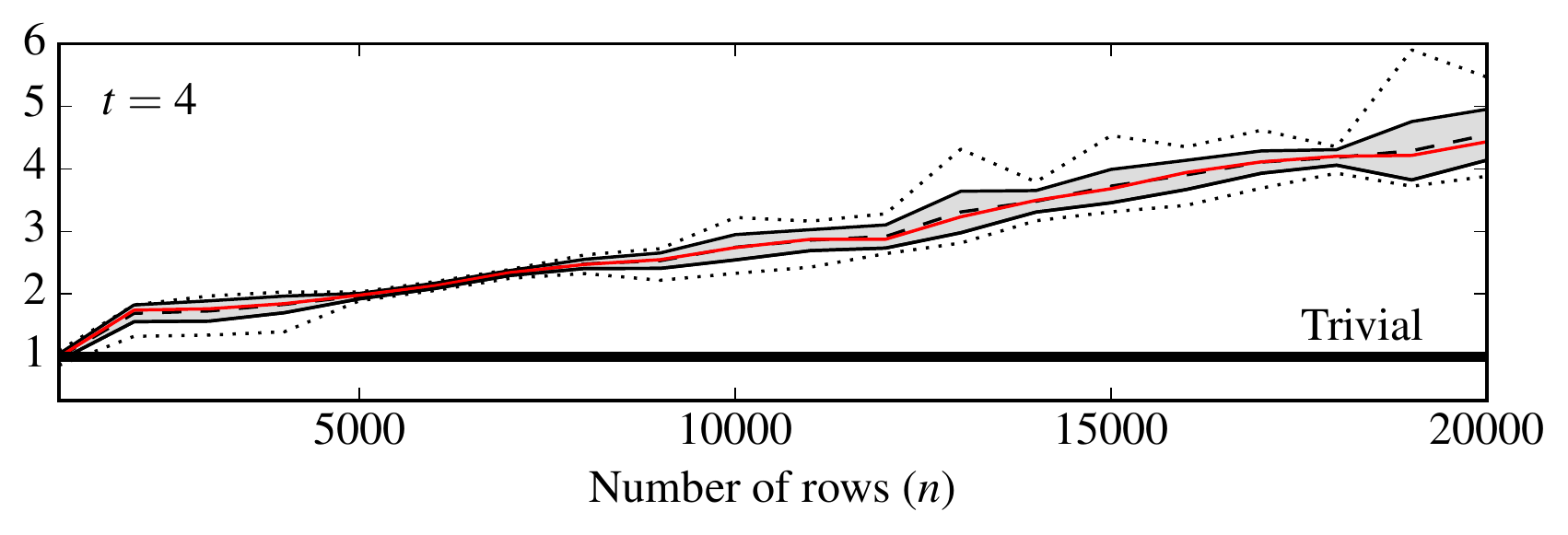}\\[-8mm]
	\caption{\label{fig:throughput mul}Relative throughput of Algorithm~\ref{algo:tablet algo simplified} when compared to the trivial quadratic algorithm on matrices of size $n\times t$ for $t=2,3,4$. A higher throughput implies faster computation. Legend: \textbf{dotted black} lines denote min and max values, \textbf{red solid} lines denote median values, \textbf{black dashed} lines denote mean values $(\mu)$, and \textbf{gray shading} denotes the range $\mu\pm\sigma$, where $\sigma$ is the standard deviation.}
\end{figure}

\emph{How does the complete \textsc{gdfv} algorithm compare with the Viterbi algorithm?} We experimentally evaluate the first phase of the $\textsc{gdfv}$ algorithm, i.e. the computation of the $q_i(s)$ values defined in Equation~\ref{eq:qis}. This is the most expensive task in the decoding of HMMs. We implement the algorithm as described in the proof of Theorem~\ref{thm:gdfv}, using $\alpha = 0.25$, that is splitting the $n\times n$ transition probability matrix $\mathcal{T}$ of the model in approximately $n/2$ blocks when $n \le 4000$. We summarize the results in Figure~\ref{fig:throughput gdfv}, where we see that our algorithm is roughly twice as fast as the Viterbi algorithm, in line with expectations. However, we note that the amount of memory required by our algorithm makes it impractical for larger values of $\alpha$. Indeed, we have verified that when the memory pressure becomes high other factors slow down the implemented algorithm, such as cache and page misses, or, for bigger allocations, the hard drive latency.

\begin{figure}[t]
	\includegraphics[width=\linewidth]{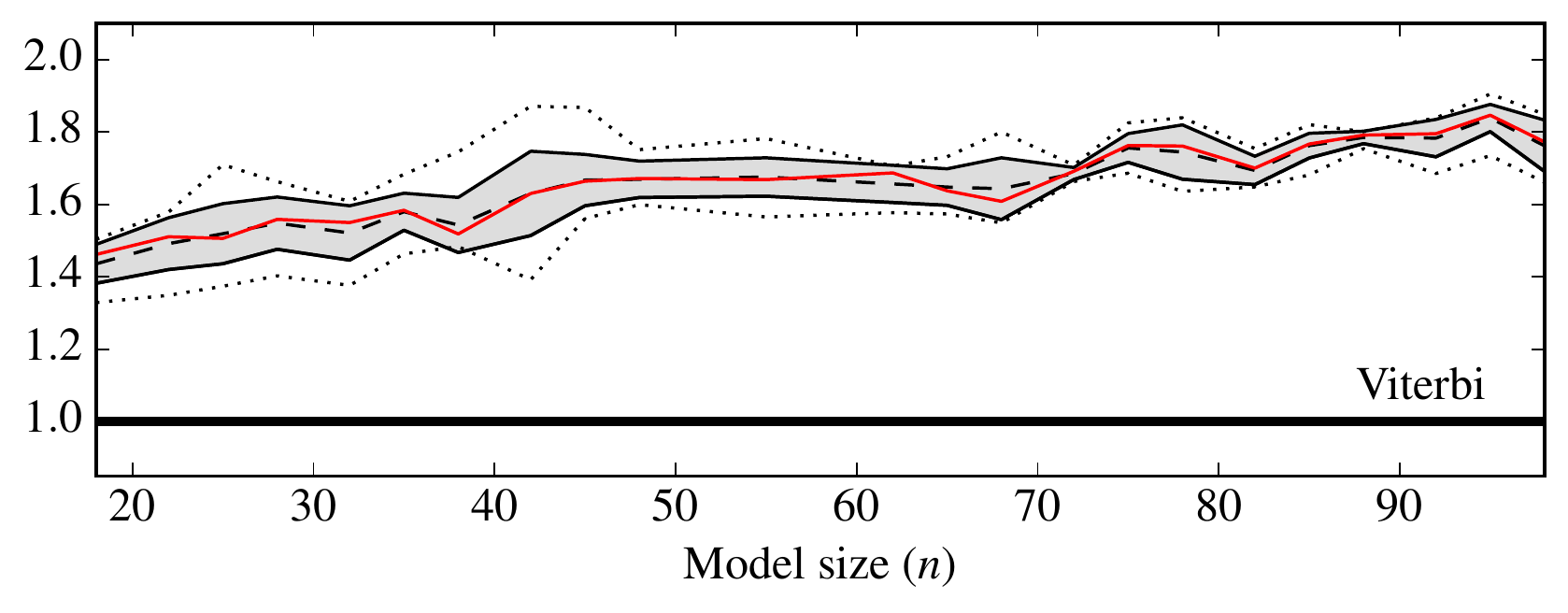}\\[-7mm]
	\caption{\label{fig:throughput gdfv}Relative throughput of the \textsc{gdfv} algorithm when compared to the Viterbi algorithm. A higher throughput implies faster computation. Legend: as in Figure~\ref{fig:throughput mul}.}
\end{figure}

\section{Conclusion and future works}
	In this paper, we give the first algorithm for the maximum \emph{a posteriori} decoding (\textsc{mapd}) of time-homogeneous Hidden Markov Models requiring asymptotically less than $\bigoh(mn^2)$ operations in the worst case. To this end, we first introduce an \emph{online} geometric dominance reporting problem, and propose a simple divide-and-conquer solution, generalizing the classical result by \cite{chan2008all}. At an intermediate step, we also give the first algorithm solving the \emph{online} matrix-vector $(\max, +)$-multiplication problem over $\Rinf$ in subquadratic time after a polynomial preprocessing of the matrix. Finally, we apply the faster multiplication to the \textsc{mapd} problem.
	
	Furthermore, we think that our proposal paves the way to several unexplored questions which we intend to explore in future works:
	\begin{itemize}[nolistsep, itemsep=0mm]
		\item cut larger polylogarithmic factors, by splitting cases in Equation~\ref{eqn:geometric running time} in a different manner, as in \cite{chan2015speeding};
		\item study and implement a more succinct version of the decision tree, in order to mitigate the memory footprint;
		\item analyze the relationship of our work with other existing
		heuristics, such as CarpeDiem \cite{esposito2009carpediem};
		\item {combine our speed-up to the one delivered by the approach in \cite{lifshits2009speeding}. Notice that this would require further assumptions on the observation sequence;}
		\item we note that the decision trees built by our algorithm could be implemented at a hardware level, resulting in specialized chips performing asymptotically less that $\bigoh(n^2)$ operations per observed symbol, in the worst case. 
		\item investigate ``truly polynomially subquadratic'' solutions for the \textsc{mapd} problem, at the expense of an exponential preprocessing of the model. As a concrete example, we present the following theorem, which is a corollary of Theorem~\ref{thm:expdominance} (a formal proof can be found in the Appendix):
		\\[-3mm] 
		\begin{restatable}{thm}{thmexphmm}\label{thm:exphmm}
			The \textsc{mapd} problem on time--homogeneous HMMs can be solved in $\bigoh(mn^{3/2}\sqrt{\log n})$ time with an $\bigoh\left(n^{1+\sqrt{n/\log n}}\right)$ time and space preprocessing.
		\end{restatable}
	\end{itemize}
	

\section{Acknowledgements}
Massimo Cairo was supported by the Department of Computer Science, University of Verona under PhD grant ``Computational Mathematics and Biology''.

We would like to thank Marco Elver, Nicola Gatti, Zu Kim, and Luigi Laura for their valuable suggestions.
\bibliography{citations}
\bibliographystyle{aaai}

\ifshowsupplementalmaterial
\clearpage\newpage
\appendix
\section{Appendix}
\vspace{3mm}

\section{Missing proofs}
We prove Theorem~\ref{thm:expdominance} and Theorem~\ref{thm:exphmm}.
\thmexpdominance*
\begin{proof}
	For every coordinate $k=1,\dots,d$, define the sequence $\b{b}^k_1, \dots, \b{b}^k_{|\Bcal|}$ containing all the vectors in $\Bcal$ ordered by their $k$--th coordinate. Namely:
	\[
		\b{b}^k_1[k] \leq \b{b}^k_2[k] \leq \dots \leq \b{b}^k_{|\Bcal|}[k].
	\]
	Given an input vector $\b{p}$, let $r_k\in \{0,\dots,|\Bcal|\}$ be the last position in which $\b{p}$ can be inserted in the sequence $\b{b}^k_1, \dots, \b{b}^k_{|\Bcal|}$, while maintaining its ordering according to the $k$--th coordinate. That is, $r_k$ is the only index that satisfies:
	\[
		\b{b}^k_1[k], \dots, \b{b}^k_{r_k}[k] \;\leq\; \b{p}[k] \;<\; \b{b}^k_{r_k + 1}[k], \dots, \b{b}^k_{|\Bcal|}[k].
	\]
	Now, observe that the solution to a query can be obtained as:
	\[
		\begin{aligned}
			\delta_\Bcal(\b{p}) &= \{ \b{b} \in \Bcal : \b{b}[k] \leq \b{p}[k] \quad \forall k \} \\
			&= \bigcap_k \{ \b{b}^k_1, \cdots, \b{b}^k_{r_k} \}.
		\end{aligned}
	\]
	The algorithm works as follows.
	At preprocessing time, we order the vectors according to each coordinate $k=1,\dots,d$. Then, we precompute and store in a lookup table the solution $\delta_\Bcal(\b{p}) = \bigcap_k \{ \b{b}^k_1, \cdots, \b{b}^k_{r_k} \}$ for each of the $|\Bcal+1|^d$ possible choices of the values $r_1,\dots,r_d$. This takes $\bigoh(|\Bcal|^{d+1})$ time and space, and dominates the cost of the preprocessing phase.

	Once the input vector $\b{p}$ is received, $r_k$ is computed  by binary search for each coordinate $k=1,\dots,d$. Then, the solution for the values $r_1,\dots,r_d$ just obtained is looked up in the table. The time for lookup is dominated by the $\bigoh(d\log |\Bcal|)$ time of the binary searches, yielding a total running time of $\bigoh(d\log |\Bcal|+|\delta_\Bcal(\b{p})|)$ to compute and return the output.
\end{proof}

\thmexphmm*
\begin{proof}
	We apply consecutively Theorem~\ref{thm:expdominance}, Lemma~\ref{lem:dominance to product}, Lemma~\ref{lem:splice} and Lemma~\ref{lem:product to viterbi}.
	First, by Theorem~\ref{thm:expdominance}, we solve the online dominance reporting problem in $\bigoh(d\log |\Bcal|)$ time.
	Second, we apply Lemma~\ref{lem:dominance to product} to this algorithm, and obtain a solution for the online matrix--vector $(\max,+)$-multiplication for a rectangular $m\times t$ matrix running in $\bigoh(m+t^2 \log |m|)$ time.
	Third, by Lemma~\ref{lem:splice}, this algorithm can be used to solve the multiplication for $n\times n$ matrices in $\bigoh(n/t \cdot (n+t^2\log n))$ time for any choice of $t$.
	Choosing $t = \sqrt{n/\log n}$, the running time is $\bigoh(n\sqrt{n \log n})=\bigoh(n^{3/2} \log^{1/2} n)$ with a preprocessing cost of $\bigoh(n^{\sqrt{n/\log n} + 1})$ time and space.
	Finally, by Lemma~\ref{lem:product to viterbi}, we employ this algorithm to solve the \textsc{mapd} problem in $\bigoh(mn^{3/2}\log^{1/2} n)$ time, with $\bigoh(n^{1+\sqrt{n/\log n}})$ time and space required for the preprocessing.
\end{proof}

\newpage
\section{Experimental evaluation}
	\begin{figure}[h!]
		\includegraphics[width=1\linewidth]{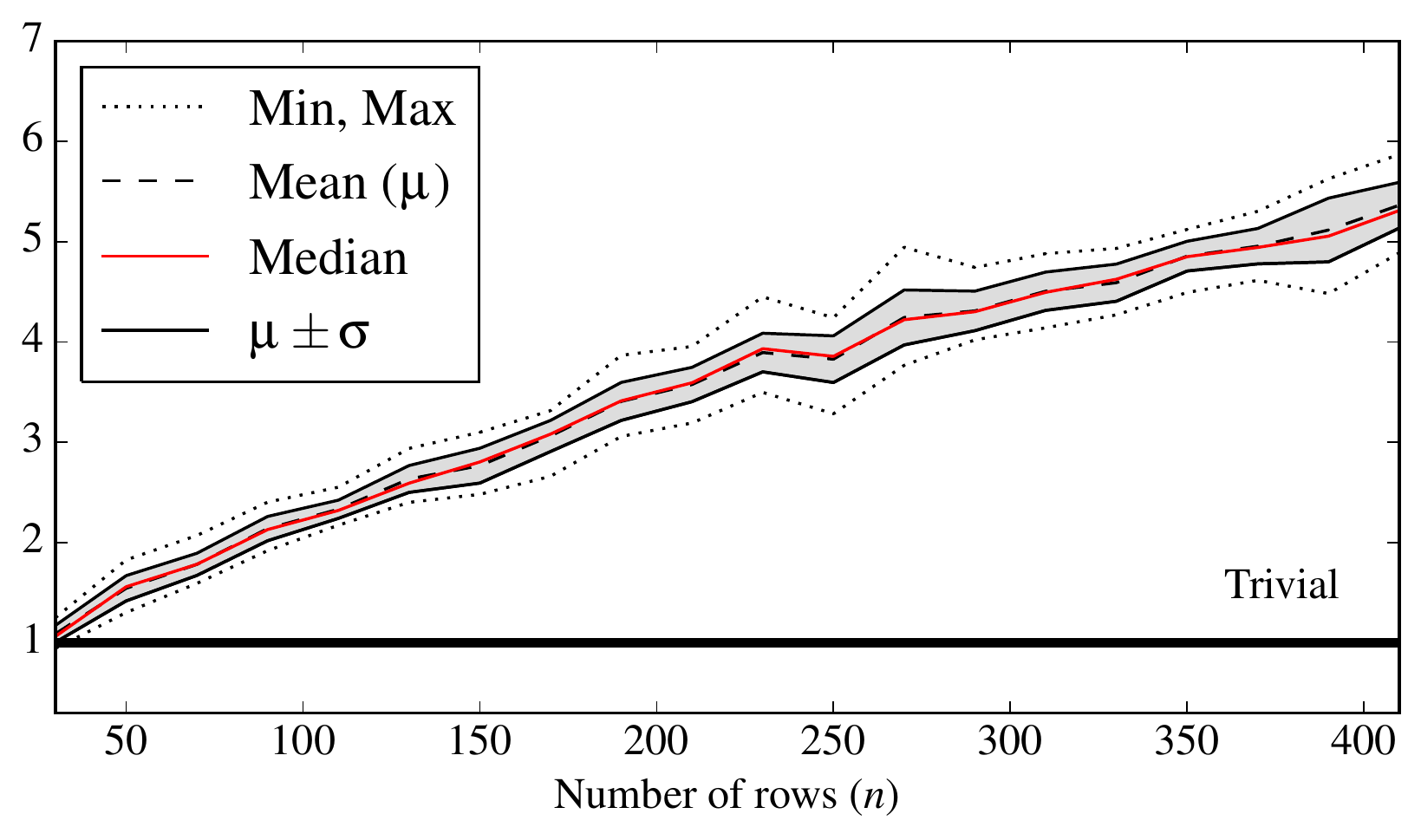}\\[-.8cm]
		\caption{Relative throughput of Algorithm~\ref{algo:tablet algo simplified} when compared to the trivial quadratic algorithm on matrices of size $n\times 2$. A higher throughput implies faster computation.}
	\end{figure}
	\vspace{-.3cm}
	\begin{figure}[h!]
		\includegraphics[width=1\linewidth]{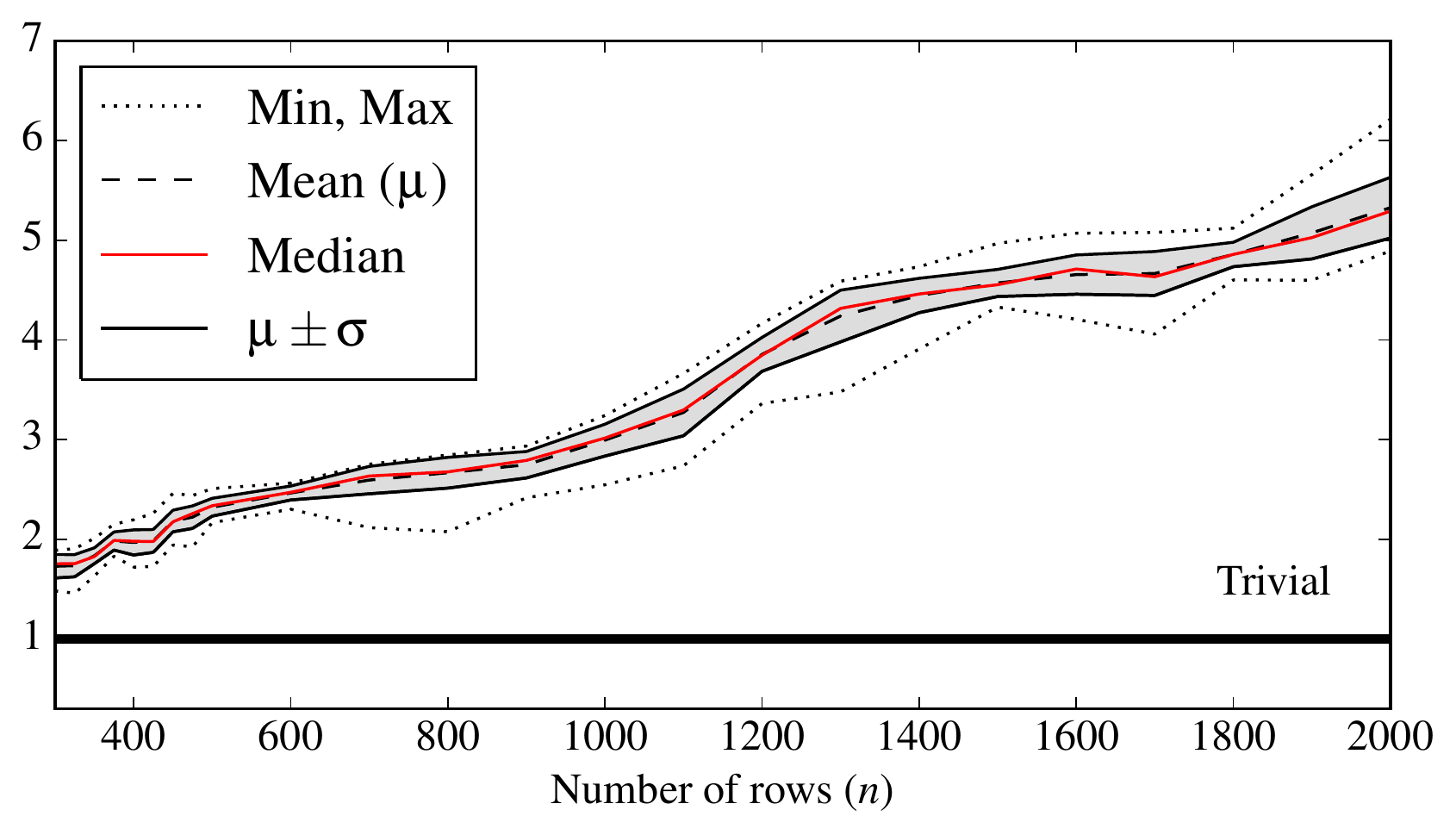}\\[-.8cm]
		\caption{Relative throughput of Algorithm~\ref{algo:tablet algo simplified} when compared to the trivial quadratic algorithm on matrices of size $n\times 3$. A higher throughput implies faster computation.}
	\end{figure}
	\vspace{-.3cm}
	\begin{figure}[h!]
		\includegraphics[width=1\linewidth]{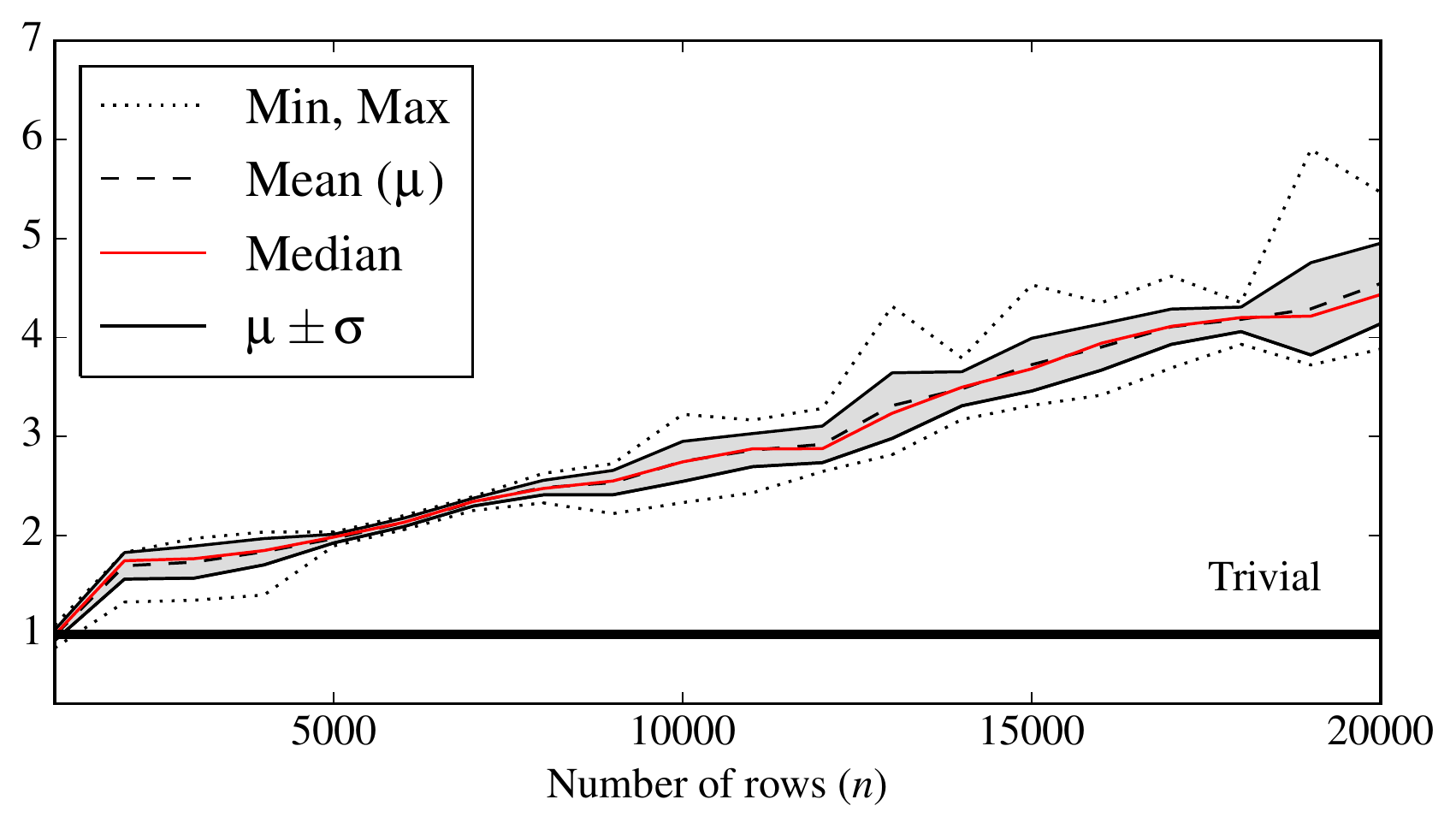}\\[-.8cm]
		\caption{Relative throughput of Algorithm~\ref{algo:tablet algo simplified} when compared to the trivial quadratic algorithm on matrices of size $n\times 4$. A higher throughput implies faster computation.}
	\end{figure}
	\clearpage\newpage

\fi
\end{document}